\documentclass[letterpaper]{amsart}
\usepackage{amsmath,amssymb,amsxtra,amsthm,hyperref}
\usepackage{amsmath,amssymb,amsfonts}
\usepackage{algorithmic}
\usepackage{graphicx}
\usepackage{textcomp}
\usepackage{xcolor}
\usepackage{amsmath,amssymb,amsxtra,amsthm,hyperref} 
\usepackage{dsfont}
\usepackage{subcaption}
\usepackage{graphicx} 

\newcommand{\Bmath}[1]{\mbox{\bf {#1}}}

\newcommand{\Bc}{\Bmath{c}}

\newcommand{\BD}{\Bmath{D}}

\newcommand{\BI}{\Bmath{I}}

\newcommand{\BT}{\Bmath{T}}
\newcommand{\Bx}{\Bmath{x}}
\newcommand{\BX}{\Bmath{X}}

\newcommand{\By}{\Bmath{y}}

\newcommand{\Bw}{\Bmath{w}}

\newcommand{\BW}{\Bmath{W}}

\newcommand{\BV}{\Bmath{V}}
\newcommand{\BU}{\Bmath{U}}
\newcommand{\Bu}{\Bmath{u}}
\newcommand{\Bv}{\Bmath{v}}

\newcommand{\BF}{\Bmath{F}}

\newcommand{\BM}{\Bmath{M}}

\newcommand{\BQ}{\Bmath{Q}}
\newcommand{\Bm}{\Bmath{m}}
\newcommand{\BJ}{\Bmath{J}}

\newcommand{\BC}{\Bmath{C}}

\newcommand{\A}{{\cal A}}

\def\n{{\Bmath n}}

\def\I{{\Bmath I}}

\def\A{{\Bmath A}}

\def\C{{\Bmath C}}

\def\W{{\Bmath W}}

\newtheorem{theorem}{Theorem}[section]

\newtheorem{corollary}[theorem]{Corollary}

\newtheorem{lemma}[theorem]{Lemma}

\theoremstyle{definition}
\newtheorem{definition}[theorem]{Definition}

\newtheorem{remark}[theorem]{Remark}

\numberwithin{equation}{section}

\usepackage{xcolor}
\usepackage{soul}

\newcommand{\norm}[1]{\left\lVert#1\right\rVert}

\usepackage{enumitem}

\usepackage{mathdots}
\usepackage{MnSymbol}

\usepackage{soul}

\begin{document}


\title{
Linear Simple Cycle Reservoirs at the edge of stability perform Fourier decomposition of the input driving signals
} 



\author{Robert Simon Fong}
\address{School of Computer Science, University of Birmingham, Birmingham, B15 2TT, UK}
\email{r.s.fong@bham.ac.uk}
\author{Boyu Li}
\address{Department of Mathematical Sciences, New Mexico State University, Las Cruces, New Mexico, 88003, USA}
\email{boyuli@nmsu.edu}
\author{Peter Ti\v{n}o}
\address{School of Computer Science, University of Birmingham, Birmingham, B15 2TT, UK}
\email{p.tino@bham.ac.uk}


\date{\today}

\begin{abstract}

This paper explores the representational structure of linear Simple Cycle Reservoirs (SCR) operating at the edge of stability. We view SCR as providing in their state space feature representations of the input-driving time series. By endowing the state space with the canonical dot-product, we ``reverse engineer" the corresponding kernel (inner product) operating in the original time series space. The action of this time-series kernel is fully characterized by the eigenspace of the corresponding metric tensor.
We demonstrate that when linear SCRs are constructed at the edge of stability, the eigenvectors of the time-series kernel align with the Fourier basis. This theoretical insight is supported by numerical experiments.

\end{abstract}

\maketitle 


\textbf{
Recurrent Neural Networks (RNNs) are machine learning methods for modeling temporal dependencies in sequential data, but their training can be computationally demanding. Reservoir Computing (RC), a simplified subset of RNNs, circumvents this problem by fixing the internal dynamics of the network (the reservoir) and focusing training on the readout layer. Simple Cycle Reservoir (SCR) is a type RC model that stands out for its minimalistic design and proven capability to universally approximate a wide class of processes operating on time series data (namely time-invariant fading memory filters) even in the linear dynamics regime (and non-linear static readouts). 
We show that, interestingly, when linear SCR is constructed at the edge of stability, it implicitly represents the time series according to a well known and widely used technique of Fourier signal decomposition.
This insight demonstrates that deep 
connections can exist between recurrent neural networks, classical signal processing techniques and statistics, paving the way for their enhanced understanding and innovative applications.
}

\section{Introduction}

Recurrent Neural Networks (RNNs) are input-driven parametric state-space machine learning models designed to capture temporal dependencies in sequential input data streams. They encode time series data into a latent state space, dynamically storing temporal information within state-space vectors.

Reservoir Computing (RC) models is a subset of RNNs that operate with a fixed, non-trainable input-driven dynamical system (known as the {\em reservoir}) and a static trainable readout layer producing model responses based on the reservoir activations. This design uniquely simplifies the training process by concentrating adjustments solely on the readout layer (thus avoiding back-propagating the error information backwards through time), leading to enhanced computational efficiency. The simplest implementations of RC models include Echo State Networks (ESNs) \cite{Jaeger2001, Maass2002, Tino2001, Lukoservicius2009}. 

Simple Cycle Reservoirs (SCR) represent a specialized class of RC models characterized by a single degree of freedom in the reservoir construction (modulo the state space dimensionality), structured through uniform ring connectivity and binary input weights with an aperiodic sign pattern. Recently, SCRs were shown to be universal approximators of time-invariant dynamic filters with fading memory over $\mathbb{C}$ and $\mathbb{R}$ respectively in \cite{li2023simple, fong2024universality}, making them highly suitable for integration in photonic circuits for high-performance, low-latency processing \cite{bauwens2022using, larger2017high, harkhoe2020demonstrating}. 


Understanding the intricacies of SCRs in depth is essential. In this work, we employ the kernel view of linear ESNs introduced in \cite{tino2020dynamical}, in which the state-space `reservoir' representation of (potentially left-infinite) input sequences is treated as a feature map corresponding to the given reservoir model through the associated reservoir kernel. For linear reservoirs, the canonical dot product of two input sequences’ feature representations is analytically expressible as the (semi-)inner product of the sequences themselves. The corresponding metric tensor reveals the representational structure imposed by the reservoir on the input sequences, in particular in the metric tensor's eigenspace containing dominant projection axes (time series 'motifs') and scaling  (`importance' factors).

To assess the ``richness" of linear SCR state-space representations, \cite{tino2020dynamical} proposed analyzing the relative area of motifs under the Fast Fourier Transform (FFT). It was observed that the richness of these representations collapses at the edge of stability when the spectral radius $\rho$ of the dynamic coupling matrix equals to 1.

In this paper we theoretically analyze the collapse of motif richness at the edge of stability and show that 
when $\rho = 1$ the SCR kernel motifs  correspond to Fourier basis. We begin by reviewing the notion of SCRs \cite{rodan2010minimum}, kernel view of ESNs \cite{tino2020dynamical}, and Reservoir Motif Machines \cite{tino2024} in Section \ref{sec:linear_motifs}. The contribution of this paper, in the subsequent sections, are outlined as follows:

\begin{enumerate}
    \item In Section~\ref{sec:harmonic}, we show in $\mathbb{C}$ that motifs of linear SCR at the edge of stability are harmonic functions. 
    \item In Section~\ref{sec:motifcount}, we show in $\mathbb{R}$ that $n$ dimensional linear SCR has $\lceil \frac{n}{2} \rceil$ symmetric motifs and $\lfloor \frac{n}{2} \rfloor$  skew-symmetric motifs.
    \item In Section~\ref{sec:fourier_basis}, we combine the results of the previous two sections and demonstrate numerically that in $\mathbb{R}$, the motifs of linear SCR at edge of stability are exactly the columns of real Fourier basis matrix. 
    \item Finally in Section~\ref{sec:exp}, we conclude the paper with numerical experiments supporting our findings. 
\end{enumerate}

\section{Simple Cycle Reservoir and its temporal kernel}
\label{sec:linear_motifs}

Let $\mathbb{K} = \mathbb{R}$ or $\mathbb{C}$ be a field. We first formally define the principal object of our study - parametrized linear driven dynamical system with a (possibly) non-linear readout.

\begin{definition}\label{def.lrc} A \textbf{linear reservoir system} over $\mathbb{K}$ is 
is the triplet $R:= (\BW,\BV,h)$ where the \textbf{dynamic coupling} $\BW \in \mathbb{M}_{n\times n}\left(\mathbb{K}\right)$ is an $n\times n$ matrix over $\mathbb{K}$, the \textbf{input-to-state coupling} $\BV \in \mathbb{M}_{n\times m}\left(\mathbb{K}\right)$ is an $n\times m$ matrix, and the state-to-output mapping (\textbf{readout}) $h:\mathbb{K}^n \to \mathbb{K}^d$ is a (trainable) continuous function. 

The corresponding linear dynamical system is given by:
\begin{equation} \label{eq.system}
   \begin{cases} \Bx_t &= \BW \Bx_{t - 1} + \BV \Bc_t \\
    \By_t &= h(\Bx_t)
    \end{cases}
\end{equation}
where $\{\Bc_t\}_{t\in\mathbb{Z}_-} \subset \mathbb{K}^m$, $\{\Bx_t\}_{t\in\mathbb{Z}_-} \subset \mathbb{K}^n$, and $\{\By_t\}_{t\in\mathbb{Z}_-} \subset \mathbb{K}^d$ are the external inputs, states and outputs, respectively.
We abbreviate the dimensions of $R$ by $(n,m,d)$.

We make the following assumptions for the system:
\begin{enumerate}
    \item $W$ is assumed to be strictly \textbf{contractive}. In other words, its operator norm $\norm{W}<1$. The system \eqref{eq.system} thus satisfies the fading memory property (FMP) \cite{li2023simple}.
    \item We assume the input stream is $\{\Bc_t\}_{t\in\mathbb{Z}_-}$ is \textbf{uniformly bounded}. In other words, there exists a constant $M$ such that $\norm{\Bc_t}\leq M$ for all $t\in\mathbb{Z}_-$. 
\end{enumerate}
 {The contractiveness of $W$ and the uniform boundedness of input stream imply that the images 
 $\Bx \in \mathbb{K}^n$ of the inputs $\Bc \in (\mathbb{K}^m)^{\mathbb{Z}_-}$ under the linear reservoir system live in a compact space $X \subset{\mathbb{K}}^n$. With slight abuse of mathematical terminology we call $X$ a \textbf{state space}.}
\end{definition}

\begin{definition} Let $\C=[c_{ij}]$ be an $n\times n$ matrix. We say $\C$ is a \textbf{permutation matrix} if there exists a permutation $\sigma$ in the symmetric group $S_n$ 
such that 
\[c_{ij}=\begin{cases} 1, &\text{ if }\sigma(i)=j, \\0, &\text{ if otherwise.}\end{cases}\]
We say a permutation matrix $\C$ is a \textbf{full-cycle permutation}\footnote{Also called left circular shift or cyclic permutation in the literature.} if its corresponding permutation $\sigma\in S_n$ is a cycle permutation of length $n$.  Finally, a matrix $\BW=\rho \cdot \C$ is called a \textbf{contractive full-cycle permutation} if $\rho\in(0,1)$ and $\C$ is a full-cycle permutation. 
\end{definition}

The idea of simple cycle reservoir was presented in \cite{rodan2010minimum} as a reservoir system with a very small number of design degrees of freedom, yet retaining performance capabilities of more complex or (unnecessarily) randomized constructions. In fact,
it can be shown that even with such a drastically reduced design complexity, SCR models are universal approximators of fading memory filters \cite{li2023simple, fong2024universality}.

\begin{definition}
\label{def.rc}
    A linear reservoir system $R = \left(\BW,\Bw,h\right)$ with dimensions $(n,m,d)$ is called a \textbf{Simple Cycle Reservoir (SCR) }
    \footnote
     {We note that the assumption on the aperiodicity of the sign pattern in $V$ is not required for this study}
    if 
    \begin{enumerate}
        \item $\BW$ is a contractive full-cycle permutation, and
        \item $\Bw \in \mathbb{M}_{n \times m}\left(\left\{-1,1\right\}\right)$. 
    \end{enumerate}
\end{definition}

One possibility to understand inner representations of the input-driving time series forming inside the reservoir systems is to view the reservoir state space as a temporal feature space of the associated reservoir kernel \cite{tino2020dynamical,gonon2022reservoir} 
.
Consider a linear reservoir system $R = (\BW,\Bw,h)$ over $\mathbb{K}$ with dimensions $(n,1,d)$ operating on univariate input. 


Let $\tau > n$ denote the length of the look back window and consider two sufficiently long time series of length $\tau>n$,

\begin{align*}
\Bu &= \left(u\left(-\tau + 1\right), u\left(-\tau + 2\right),\ldots,u\left(-1\right),u\left(0\right)\right)\\
    &=: \left(u_1,u_2,\ldots,u_\tau\right) \in \mathbb{K}^\tau
\end{align*}
and
\begin{align*}
\Bv &= \left(v\left(-\tau + 1\right), v\left(-\tau + 2\right),\ldots,v\left(-1\right),v\left(0\right)\right)\\
    &=: \left(v_1,v_2,\ldots,v_\tau\right) \in \mathbb{K}^\tau
\end{align*}
we consider the reservoir states reached upon reading them (with zero initial state) their feature space representations \cite{tino2020dynamical}:
\[
\phi(\Bu) = \sum_{j=1}^\tau u_j \BW^{\tau - j}\Bw, \ \ \ 
\phi(\Bv) = \sum_{j=1}^\tau v_j \BW^{\tau - j}\Bw.
\]
The canonical dot product (reservoir kernel)
\[
K(\Bu,\Bv) = \langle \phi(\Bu),\phi(\Bv) \rangle
\]
can be written in the original time series space as a semi-inner product
$\langle \Bu,\Bv \rangle_{Q} = \Bu^\top \BQ \Bv$, where
\begin{equation}
\label{eqn:Q}
Q_{i,j} = \Bw^\top \left(\W^\top{}\right)^{i-1} \BW^{j-1} \Bw.
\end{equation}

Since the (semi-)metric tensor $\BQ \in \mathbb{M}_{\tau\times\tau}(\mathbb{K})$ is symmetric and positive semi-definite, it admits the following eigen-decomposition:
\begin{align}
\label{eqn:m}
    \BQ = \BM \Lambda_Q \BM^\top ,
\end{align}
where 
$\Lambda_Q := \operatorname{diag}\left(\lambda_1,\lambda_2,\ldots,\lambda_{N_m} \right)$ 
is a diagonal matrix consisting of non-negative eigenvalues of $\BQ$ with the corresponding eigenvectors $\Bm_1,\Bm_2,\ldots,\Bm_{N_m} \in \mathbb{K}^\tau$
(columns of $\BM$). The $N_m := \operatorname{rank}\left(Q\right) \leq N \leq \tau$ eigenvectors of $\BM$ with positive eigenvalues are called the {\em motifs} of $R$. 
We have:
\[
K(\Bu,\Bv) =
\left(\Lambda_Q^{\frac{1}{2}} \BM^\top \Bu\right)^\top \left(\Lambda_Q^{\frac{1}{2}} \BM^\top \Bv\right).
\]

In particular, the reservoir kernel is a canonical dot product 
of time series projected onto the motif space spanned by $\left\{\Bm_i\right\}_{i=1}^{N_m}$:
\[
K(\Bu,\Bv) =
\langle \tilde \Bu, \tilde \Bv \rangle,
\]
where
\begin{align}
\label{eqn:rc_filterview}
\tilde \Bu &= 
\Lambda_Q^{\frac{1}{2}} \BM^\top \Bu \\ \nonumber 
&= 
\begin{bmatrix}
     \lambda^{\frac{1}{2}}_1 \cdot \langle \Bm_1, \Bu \rangle \\
     \vdots \\
     \lambda^{\frac{1}{2}}_{N_m} \cdot \langle \Bm_{N_m}, \Bu\rangle 
\end{bmatrix} \\ \nonumber 
& =
\left(\lambda_i^{\frac{1}{2}} \cdot \langle \Bm_i,\Bu \rangle \right)_{i=1}^{N_m} \in \mathbb{M}_{N_m}(\mathbb{K}).
\end{align}

\textbf{Reservoir Motif Machine} (RMM)\cite{tino2024} is a predictive model motivated by the kernel view of linear Echo State Networks described above. By projecting the $\tau$-blocks of input time series onto the reservoir motif space given by $\operatorname{span}(\{\Bm_i \}_{i=1}^{N_m})$, RMM captures temporal and structural dynamics in a computationally efficient feature map.

In particular, rather than relying on the motif weights $\lambda_i^{\frac{1}{2}}$ determined by the reservoir $R$ in Equation~\eqref{eqn:rc_filterview}, RMM introduces a set of adaptable motif coefficients, denoted as $C := \{ c_i \in \mathbb{R} \}_{i=1}^{N_m}$, to define its feature map as follows:

\begin{align}
\label{eqn:conv_layer}
\varphi\left(\Bu;C\right)
&= 
\begin{bmatrix}
     c_1 \cdot \langle \Bm_1, \Bu \rangle \\
     \vdots \\ \nonumber
     c_{N_m} \cdot \langle \Bm_{N_m}, \Bu\rangle 
\end{bmatrix} \\
& =
\left(c_i \cdot \langle \Bm_i,\Bu \rangle \right)_{i=1}^{N_m} \in \mathbb{R}^{N_m}.
\end{align}

This feature map is used to train a predictive model, such as linear regression or kernel-based methods, directly in the motif space.

\begin{remark}
    \label{rmk:ntau}
To streamline the theoretical analysis of SCR kernels, in line with \cite{tino2020dynamical}, we assume that the length of the look-back window (past horizon) $\tau$ is an integer multiple of the dimension of the state space $n$, i.e. there exists $k \in \mathbb{N}_+$ such that $\tau = k\cdot n$. In particular, denoting by $\Bm$ a SCR motif calculated with $\tau = n$, it is shown in \cite{tino2020dynamical} that when 
$\tau = k\cdot n$, the corresponding motif
is a concatenation of $k$ copies of $\Bm$, scaled by $\rho^{l\cdot n}$, $l=0,1,2,...,k-1$. Recall that $\rho$ is the spectral radius of the dynamic coupling $W$. Hence, to study SCR motifs with past horizon $\tau = k\cdot n$, for any $k \in \mathbb{N}_+$, it is sufficient to study only the base case $\tau = n$.
\end{remark}

\section{On the edge of stability of Simple Cycle Reservoirs}

Having defined the reservoir temporal kernel, one can ask how ``representationally rich'' is the associated feature space (span of the reservoir motifs).  To quantify the `richness' of the reservoir feature space 
of a linear reservoir system over $\mathbb{R}$, \cite{tino2020dynamical} proposed the following procedure:

Consider a linear reservoir system $R = (\BW,\Bw,h)$ with dimensions 
$(n,1,d)$ over $\mathbb{R}$. Suppose $\BW$ has spectral radius $\rho$.
Recall from Remark~\ref{rmk:ntau} that: to study the SCR motif structure of, it is sufficient to consider the past horizon $\tau = n$.The motif matrix $\BM \in \mathbb{M}_{n\times n}\left(\mathbb{R}\right)$ is constructed according to Equation~\eqref{eqn:m} from the matrix $\BQ$ (metric tensor of the inner product of the reservoir kernel).

First, Fast Fourier Transform (FFT) is applied to the kernel motifs (columns of $\BM$), considering only those with motif weights upto a threshold of $10^{-2}$ of the highest motif weight. This yields an $n \times n'$ matrix of Fourier coefficients over $\mathbb{C}$ with $n' \leq n$. These Fourier coefficients are then collected in a (multi)set $Z:= \left\{z_k\right\}_{k=1}^{n \cdot n'} $.

To evaluate the diversity and spread of the Fourier coefficients in the complex plane, \cite{tino2020dynamical} proposed calculating the coarse-grained area occupied by $Z \subseteq \mathbb{C}$. In particular, 
the box $\left[-7, 7 \right]^2$ in the complex plane is partitioned into a grid of cells (following \cite{tino2020dynamical} we use  side length $0.05$). The \textbf{relative area} covered by $Z \subseteq \mathbb{C}$ is defined as the ratio of the number of cells visited by the coefficients $z_k \in Z$ to the total number of cells in the grid. An example of the distribution of Fourier coefficients $Z$ of linear SCR with $n = 97$ at $\rho = 0.9813, 0.999, 1$ and $\Bw$ being the first $n$ digits of binary expansion of $\pi$ is presented in Figure~\ref{fig:fourier_coeff} \footnote{In particular, $\rho = 0.9812798473475446$ is where the relative area peaks at Figure~\ref{fig:relative_area}.}. 

\begin{figure}[ht!]
    \centering
    \begin{subfigure}{0.3\textwidth}
        \centering
        \includegraphics[width=1.1\linewidth]{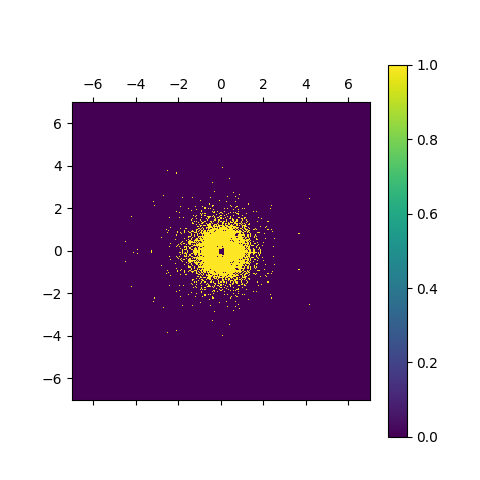}
        \caption{$\rho = 0.9812798473475446$}
    \end{subfigure}
    \hspace{0.01\textwidth} 
    \begin{subfigure}{0.3\textwidth}
        \centering
        \includegraphics[width=1.1\linewidth]{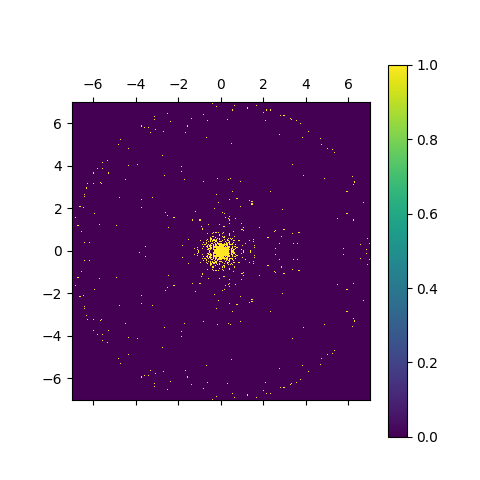}
        \caption{$\rho = 0.999$}
    \end{subfigure}
    \hspace{0.01\textwidth} 
    \begin{subfigure}{0.3\textwidth}
        \centering
        \includegraphics[width=1.1\linewidth]{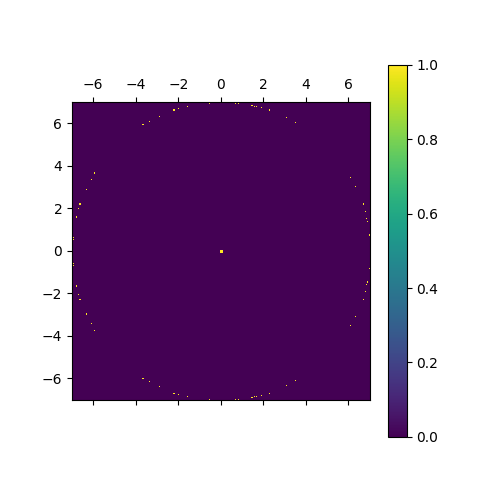}
        \caption{$\rho = 1$}
    \end{subfigure}
    \caption{Example of Fourier coefficient of linear SCR at $\rho = 0.9812798473475446, 0.999, 1$ respectively. $0.9812798473475446$ in particular is where the relative area peaks in Figure~\ref{fig:relative_area}.}
    \label{fig:fourier_coeff}
\end{figure}

\begin{figure}[ht!]
\centering
\includegraphics[scale = 0.9]{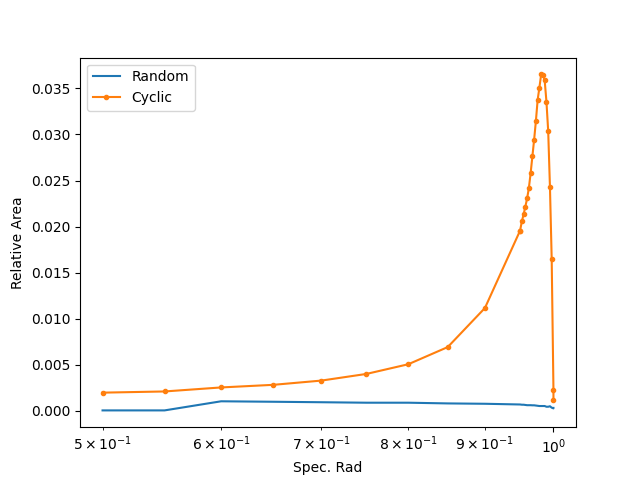}
\caption{Relative area of Linear SCR and Randomly generated Reservoir with respect to the spectral radius}
\label{fig:relative_area}
\end{figure}

Replicating the experiment in \cite{tino2020dynamical} in Figure~\ref{fig:relative_area}, we observe that the ``richness" of the motif space of SCR, as measured by relative area under the current setup, increases as the spectral radius approaches approximately $\rho \approx 0.9813$. Beyond this point, the measure sharply declines, aligning with the results for a randomly generated reservoir. This decline is also observed in  is also shown in Figure~\ref{fig:relative_area}, as $\rho$ increases from approximately $0.98$ (the peak of Figure~\ref{fig:relative_area}) to the edge of stability at $\rho =1$.

In this paper, we aim to investigate this collapse of representational richness of SCR models at the edge of stability $\rho = 1$.
In particular, we will show that at the edge of stability, 
the motifs of linear SCR are (sampled) harmonic functions.

The next two section will be a two-part exposition of the properties of the motifs of linear SCR:
\begin{enumerate}
    \item In Section~\ref{sec:harmonic}, we show in $\mathbb{C}$ that the motifs are harmonic functions. Following the approach of \cite{li2023simple}, we begin with unitary dynamical coupling and progress to cyclic permutations.
    \item In Section~\ref{sec:motifcount}, we demonstrate in $\mathbb{R}$ that the number of symmetric motifs is $\lceil \frac{n}{2} \rceil$, while the number of skew-symmetric motifs is $\lfloor \frac{n}{2} \rfloor$. In line with \cite{fong2024universality}, we start with orthogonal coupling and then move onto cyclic permutation.
\end{enumerate}

Combining these two results, we conclude that at $\rho =1$, the motifs alternate between real and imaginary components of the first Fourier basis, which correspond to cosines and sines, respectively. We supplement our theoretical findings with numerical simulations of the motif space of SCR in the next section, which then lead to the numerical experiments in the final section.

\section{Unit spectral radius SCR implies harmonic motifs in Complex Domain}
\label{sec:harmonic}

We first show that, in the complex domain $\mathbb{C}$, the motifs of SCR can be derived explicitly. In particular, in this section we set $\mathbb{K} = \mathbb{C}$ and show that when the spectral radius $\rho = 1$, the motifs of linear SCR are \textit{harmonic}, i.e. they are precisely the Fourier basis (columns of the Fourier matrix).

Consider a linear reservoir system $R = (\BW,\Bw,h)$ over $\mathbb{C}$
with dimensions $(n,1,d)$ and $\BW$ of spectral radius $\rho$. 
Let $\BQ_\rho$ denote the metric tensor of the reservoir kernel (Eq.~\eqref{eqn:Q}. 

In the spirit of \cite{li2023simple}, we begin by considering a more general setting of $\BW=\rho \BU$, where $\BU$ is a unitary matrix in $\mathbb{M}_{n\times n}(\mathbb{C})$
 (i.e. $\BU\BU^*=\BU^*\BU=\I$) and  $\rho\in(0,1]$. We then move to the special case where $\BU = \BC$ is a cyclic permutation. Since $\BU$ is unitary, its eigenvalues all have norm $1$. We let its eigenvalues be $\{s_j\in\mathbb{C}: 1\leq j\leq n\}$ with corresponding eigenvectors $\{\xi_j: 1\leq j\leq n\}$, and we know each $|s_j|=1$. 

By construction:
\[\BQ_\rho = [Q_{\rho \ (i,j)}]_{1\leq i,j\leq \tau} = [\mathbf{w}^* \BW^{*(i-1)} \BW^{j-1} \mathbf{w}]_{1\leq i,j\leq \tau}.\]

The matrix $\BQ_\rho$ is a $\tau\times \tau$ matrix. Denote
\[\BX_\rho = [\BW^{*(i-1)} \BW^{j-1}]_{1\leq i,j\leq \tau},\]
and,
\[\hat{\mathbf{w}} = \begin{bmatrix} \mathbf{w} &&& \\ & \mathbf{w} && \\ && \ddots & \\ &&& \mathbf{w}\end{bmatrix}.\]
Notice that $\BX_\rho$ is an $\tau n\times \tau n$ matrix while $\hat{\mathbf{w}}$ is an $\tau n\times \tau$ matrix. Then by construction, $\BQ_\rho = \hat{\mathbf{w}}^* \BX_\rho \hat{\mathbf{w}}$. Notice that since $\BW = \rho \BU$ and $\BU$ is unitary, we can rewrite $\BX_\rho$ as follows:
\[\BX_\rho = [\rho^{i+j-2} \BU^{j-i}]_{1\leq i,j\leq \tau} = \begin{bmatrix}
\I & \rho \BU & \rho^2 \BU^2 & \cdots & \rho^{\tau-1} \BU^{\tau-1} \\
\rho \BU^* & \rho^2 \I & \rho^3 \BU & & \vdots \\
\rho^2 \BU^{*2} & \rho^3 \BU^* & \rho^4 \I &  & \\
\vdots & & & \ddots & \\
\rho^{\tau-1} \BU^{*(\tau-1)} & & & & \rho^{2(\tau-1)}\I
\end{bmatrix}\]
Finally, we let $\lambda = 1+\rho^2 + \rho^4 + \cdots + \rho^{2(\tau-1)} = \begin{cases}
    \tau, & \text{ if } \rho=1\\
    \frac{1-\rho^{2\tau}}{1-\rho^2}, & \text{ if } \rho<1
\end{cases}.$

With the ultimate goal of computing the eigen-decomposition of $\BQ_\rho$, we first characterize the eigenvalues and eigenvectors of $\BX_\rho$. For the eigenvalues 
we first observe the following:

\begin{lemma}\label{lm.iv.1} The matrix $\BX_\rho$ satisfies $\BX_\rho^2 = \lambda \BX_\rho$.
\end{lemma}

\begin{proof} Multiplying the $i$-th row of $\BX_\rho$ with its $j$-th column, we obtain:
\[\sum_{k=0}^{\tau-1} \rho^{i+k-1} \BU^{-(i-k-1)} \cdot \rho^{j+k-1} \BU^{j-k-1} = \left(\sum_{k=0}^{\tau-1} \rho^{2k}\right) \rho^{i+j-2} \BU^{j-i} = \lambda \rho^{i+j-2} \BU^{j-i}.\]
This proves the desired equality. 
%
\end{proof}

As a result, we can now fully characterize the \textit{eigenvalues} of $\BX_\rho$. 

\begin{corollary}\label{cor.iv.2} The eigenvalues of $\BX_\rho$ are either $0$ or $\lambda$. Moreover, the multiplicity of the eigenvalue $\lambda$ is $n$.     
\end{corollary}

\begin{proof} Let $\xi$ be an eigenvector of $\BX_\rho$ with eigenvalue $\kappa$. Then $\BX_\rho \xi = \kappa \xi$, and thus 
\[\BX_\rho^2 \xi = \kappa^2 \xi = (\lambda \BX_\rho) \xi = \kappa\lambda \xi.\]
This implies that either $\kappa=0$ or $\lambda$.

Moreover, $Tr(\BX_\rho)=n(1+\rho^2+\cdots+\rho^{2(\tau-1)})=n\lambda$. But the trace also equals the sum of all the eigenvalues. Since the eigenvalues of $\BX_\rho$ can only be $0$ or $\lambda$, we conclude that the multiplicity of the eigenvalue $\lambda$ is precisely $n$.     
\end{proof}

We now turn to characterize the non-zero \textit{eigenvectors} of $\BX_\rho$. We will then use these vectors to compute the motif matrix of $\BQ_\rho$. Recall $\xi_j$ are the eigenvectors of $\BU$ with the corresponding eigenvalues $s_j$. For each $j$, define
\[
\hat{\xi}_j = \begin{bmatrix} \xi_j \\ \rho s_j^{-1} \xi_j \\ \vdots \\ \rho^{\tau-1} s_j^{-(\tau-1)} \xi_j\end{bmatrix} \in \mathbb{C}^{n \tau}
\]

\begin{lemma}\label{lm.iv.3} Each $\hat{\xi}_j$ is an eigenvector of $\BX_\rho$ with eigenvalue $\lambda$. Moreover, 
$\left\{\frac{1}{\sqrt{\lambda}}\hat{\xi}_j\right\}_{j=1}^n$ 
 is an orthonormal set. 
\end{lemma}

\begin{proof} First, since $\BU$ is unitary, we have $\BU\xi_j = s_j \xi_j$ and $\BU^*\xi_j = s_j^{-1}\xi_j$. Therefore, $\BU^k \xi_j = s_j^{k} \xi_j$ and ${\BU^*}^k \xi_j = s_j^{-k} \xi_j$
for all non-negative integers $k$, or, stated more compactly, 
$\BU^k \xi_j = s_j^{k} \xi_j$ for all integers $k$. 

Expanding $\BX_\rho \hat{\xi}_j$, we observe that its $i$-th block entry equals:
\[\sum_{k=0}^{n-1} \rho^{i+k-1} \BU^{-(i-k-1)} \rho^k s_j^{-k} \xi_j = \sum_{k=0}^{n-1} \rho^{i+2k-1} s_j^{-i+1} \xi_j = \left(\sum_{k=0}^{n-1} \rho^{2k} \right) \rho^{i-1} s_j^{-(i-1)}\xi_j = 
\lambda  \rho^{i-1} s_j^{-(i-1)}\xi_j.
\]
This is precisely $\lambda$ times the $i$-th block entry of $\hat{\xi}_j$.
 This proves that $\BX_\rho \hat{\xi}_j=\lambda \hat{\xi}_j$. 

Since each $\xi_j$ is a unit vector and each $|s_j|=1$, we have that $\|\hat{\xi}_j\|^2=\sum_{k=0}^{n-1} \rho^{2k} = \lambda$. Hence, $\frac{1}{\sqrt{\lambda}}\hat{\xi}_j$ is a vector of norm $1$. 
Finally, eigenvectors $\{\xi_j\}$ of any unitary matrix form an orthonormal basis, so $\langle \xi_i, \xi_j\rangle =\delta_{ij}$. This implies:
\[
\langle \hat{\xi}_i, \hat{\xi}_j\rangle = 
\sum_{k=0}^{n-1} \rho^{2k} s_i^{-k} 
(\overline{s_j})^{-k} \langle \xi_i, \xi_j\rangle,
\]
which is $0$ when $i\neq j$. This proves that 
$\left\{\hat{\xi}_j\right\}$, $j = 1,\ldots, n$, are pairwise orthogonal. 
\end{proof}

Since all the other eigenvalues of $\BX_\rho$ are $0$, we immediately have the desired eigen-decomposition of the matrix $\BX_\rho$ as:
\[\BX_\rho = \frac{1}{\lambda} \sum_{j=1}^n \hat{\xi}_j \hat{\xi}_j^*\]

Recall an $n\times n$ full-cycle permutation matrix is given by:
\[
\BC:=\begin{bmatrix}
 0        & \cdots & \cdots & \cdots & 1      \\
 1        & 0      &        & \cdots & 0 \\
 0        & 1      & 0      & \vdots & \vdots \\
 \vdots   & 0      & \ddots & 0      & \vdots \\
 0        & \vdots & 0      & 1      & 0      
\end{bmatrix}
\]

For the rest of the section, we set \textbf{$\rho=1$ and $\BW = \C$ is a full-cycle permutation} and we will now compute eigen-decomposition of the metric tensor $\BQ=\BQ_1$ of the corresponding reservoir kernel.

From elementary matrix analysis, we know the eigenvalues of a full-cycle permutation $\C$ are precisely the $n$-th root of unities $\{\omega_j = e^{\frac{2\pi i j}{n}}, 1\leq j\leq n\}$. 
Its normalized eigenvectors for each eigenvalue $\omega_j$ is given by the Fourier basis:
\begin{equation}
\label{eq:xi_j}
\xi_j = \frac{1}{\sqrt{n}}\begin{bmatrix} 1 \\ \omega_j \\ \omega_j^2 \\ \vdots \\ \omega_j^{n-1}\end{bmatrix} \in \mathbb{C}^n, \quad j = 1,\ldots, n.
\end{equation}

We will now compute the eigenvalues and eigenvectors of $\BQ$ explicitly. First, we see that
\[\BQ = \hat{\mathbf{w}}^* \BX_1 \hat{\mathbf{w}},\]
where $\BX_1= \frac{1}{n} \sum_{j=1}^n \hat{\xi}_j \hat{\xi}_j^*$. 
Define $\xi_j^* \mathbf{w} = \langle \mathbf{w}, \xi_j\rangle = d_j$. 

While, following Remark \ref{rmk:ntau}, we focused on the case where $\tau=n$, the results of this section up to this point (Lemma \ref{lm.iv.1}, Corollary \ref{cor.iv.2}, and Lemma \ref{lm.iv.3}) hold for general $\tau$.

\begin{theorem}
\label{thm:harmonic}
Consider a linear SCR system $R = (\BC,\Bw,h)$ over $\mathbb{C}$ with dimensions $(n,1,d)$ and with the full cycle permutation matrix $\BC$ as its dynamical coupling. Let $\BQ$ denote  metric tensor of the reservoir kernel under the past horizon $\tau=n$. Then, the normalized eigenvectors $\xi_j$, $j=1,2,...,n$ of $\BC$ (eq. \eqref{eq:xi_j}) are also eigenvectors of $\BQ$ with the corresponding eigenvalues equal to the squared projections of the input coupling vector $\Bw$ onto $\xi_j$, $|d_j|^2 = |\xi_j^* \mathbf{w}|^2$. In other words, $R$ has motifs 
$\xi_j$, with motif weights
$|d_j| = |\xi_j^* \mathbf{w}| = |\langle \mathbf{w}, \xi_j\rangle|$,
$j=1,2,...,n$.
\end{theorem}

\begin{proof}

Define an $n\times\tau$ matrix,
\[ \A = \frac{1}{\sqrt{n}} \begin{bmatrix} \hat{\xi}_1^* \\
\hat{\xi}_2^* \\
\vdots \\
\hat{\xi}_n^*
\end{bmatrix} \hat{\mathbf{w}}\]
Then $\BQ=\A^*\A$, and we have:
\[\A=\begin{bmatrix}
d_1 & \omega_1^{-1} d_1 & \cdots & \omega_1^{-(\tau-1)} d_1 \\
d_2 & \omega_2^{-1} d_2 & \cdots & \omega_2^{-(\tau-1)} d_2 \\
\vdots & \vdots & \vdots & \vdots \\
d_n & \omega_n^{-1} d_n & \cdots & \omega_n^{-(\tau-1)} d_n \\
    \end{bmatrix} 
 = \BD \cdot \BF
\]
Here, by the assumption that $\tau=n$, we can define two $n\times n$ matrices $\BD=\operatorname{diag}\{d_1, d_2, \cdots, d_n\}$ and 
{$\BF=[\omega_i^{-(j-1)}]$}. Note that $\BF^*$ is the discrete Fourier transform matrix, so we know $\BF^* \BF = \BF\BF^* = n\BI$. 
Therefore,
\[\BQ = \left(\frac{1}{\sqrt{n}}\BF^*\right) \operatorname{diag}\{|d_1|^2, |d_2|^2, \cdots, |d_n|^2\} \left(\frac{1}{\sqrt{n}}\BF\right).\]
Here, $(\frac{1}{\sqrt{n}}\BF)$ is a unitary matrix. The $j$-th column of $(\frac{1}{\sqrt{n}}\BF^*)$, which is precisely $\xi_j$, is the eigenvector for $|d_j|^2$
\end{proof}

\begin{remark}
    When $\tau\neq n$, we can still reach a similar decomposition as above, but the matrix $\BF=[\omega_i^{-(j-1)}]$ is an $n\times \tau$ matrix. When $\tau$ is not an integer multiple of $n$ in particular, its rows may not be orthogonal to each other in general, so it is hard to characterize the eigenvectors of $\BQ$. However, when $\tau$ is an integer multiple of $n$, the rows of $\BF$ are orthogonal to each other. In this case, the eigenvectors of $\BQ$ is given by the $n$ of these $\tau$ dimensional vectors $(\omega_i^{-j})_{j=0}^{\tau-1}$.
\end{remark}

As a result, in the case when $\rho=1$ and the matrix $\BW=\C$ is a full-cycle permutation, the motif matrix for $\BQ$ consists of $\xi_j$, which is precisely the discrete Fourier transform matrix. In other words, the columns of the motif matrix are precisely the Fourier basis. 

\begin{remark}
\label{rmk:pairs}
    Notice that $|d_j|^2 = |d_{n-j}|^2$, so that the eigenvalues and eigenvectors of the motif matrix come in pairs. Here, $\xi_j$ and $\xi_{n-j}=\overline{\xi_j}$ share the same eigenvalue $|d_j|^2=|d_{n-j}|^2$. 
\end{remark}

\section{Motifs of Orthogonal Dynamics with Unit Spectral Radius in Real Domain}
\label{sec:motifcount}

In this section we return to the real domain, i.e. $\mathbb{K} = \mathbb{R}$ and show that at unit spectral radius, the motifs of SCR consists of a fixed number of symmetric and skew symmetric vectors. As in \cite{fong2024universality}, we begin by deriving properties of the motif space of linear reservoir systems 
with
orthogonal dynamical coupling
and then move on to the special case 
of cyclic permutation.

Let $\BW \in \mathbb{M}_{n\times n}\left(\mathbb{R}\right)$ be the dynamical coupling matrix of a reservoir system $R=(\BW,\Bw,h)$ over $\mathbb{R}$. Suppose $\BW$ is orthogonal with spectral radius $\rho = 1$. We now show that the matrix corresponding to the reservoir kernel $\BQ := \BQ_{\rho = 1}$ is Toeplitz.



$\BQ$ is Toeplitz if and only if $\BQ_{ij} = \BQ_{i+1,j+1}$ for all $i,j = 1,\ldots, n-1$. By construction of $\BQ$, this is satisfied if and only if:
\[
    \Bw^\top \left(\BW^\top\right)^{i-1} \BW^{j-1} \Bw = \Bw^\top \left(\BW^\top\right)^{i} \BW^{j} \Bw .
\]
Now, orthogonality of $\BW$ implies $\BW^\top \BW = \BI$. Without loss of generality assume that $i\ge j$.
Then
\[
\left(\BW^\top\right)^{i-1} \BW^{j-1} = \left(\BW^\top\right)^{i-j}
=
\left(\BW^\top\right)^{i} \BW^{j},
\]
showing that $\BQ$ is indeed Toeplitz.


Let $\BJ$ denote the exchange matrix with $1$ on the antidiagonal and $0$ everywhere else:
\[
\BJ:=\begin{bmatrix}
0 & 0 & 0 & \cdots & 0 & 1 \\
0 & 0 & 0 & \cdots & 1 & 0 \\
0 & 0 & \cdots & 1 & 0 & 0 \\
\vdots & \vdots & \iddots & \iddots & \vdots & \vdots \\
0 & 1 & 0 & \cdots & 0 & 0 \\
1 & 0 & 0 & \cdots & 0 & 0
\end{bmatrix}
\]

\begin{lemma}
\label{lemma.centrosymmetric}
Symmetric Toeplitz square matrices are centrosymmetric.
\end{lemma}
\begin{proof}

Let $\BT$ denote a symmetric Toeplitz $\tau\times \tau$ matrix. 
Let $\left\{t_k\right\}_{k = 0}^{\tau -1}$ denote the generating sequence of $\BT$ such that $\BT$ is expressed as: 

\[\BT=\begin{bmatrix}
 t_0      & t_1     & \cdots & t_{\tau-2}& t_{\tau-1}\\
 t_1      & t_0     & \cdots & t_{\tau-3}& t_{\tau-2}\\
 t_2      & t_1     & \ddots & \vdots & \vdots \\
 \vdots   & \vdots  & \cdots & t_0    & t_1    \\
 t_{\tau-1}  & t_{\tau-2} & \cdots & t_1    & t_0 
\end{bmatrix}.
\]

Consider:
\begin{align*}
    \BT_{\tau+1-i, \tau+1-j} &= t_{\tau + 1 - i - \tau - 1 + j}\\
     &= t_{j - i} = t_{i = j} = \BT_{ij}.
\end{align*}
The shows that $\BT$ satisfies the definition of centro-symmetry. The second last equality is given by symmetry of $\BT$ and the last equality is due to $\BT$ being Toeplitz. 



\end{proof}


\begin{corollary}
    Let $\BW$ be the dynamical coupling matrix of a reservoir system $R=(\BW,\Bw,h)$. If $\BW$ is orthogonal with spectral radius $\rho = 1$, then $\BQ$ is symmetric centrosymmetric.
\end{corollary}
\begin{proof}
    $\BQ$ is symmetric by construction. Moreover, when $\BW$ is orthogonal with spectral radius $1$, $\BQ$ is Toeplitz and hence centrosymmetric by lemma \ref{lemma.centrosymmetric}.
\end{proof}

\begin{definition}
Let $\mathbb{K}$ be an arbitrary field. A $\tau$-dimensional vector $v$ over $\mathbb{K}$ is \textbf{symmetric} if:
\[
\BJ v = v.
\]
Similarly, $v$ is \textbf{skew-symmetric} if:
\[
\BJ v = -v.
\]
\end{definition}

By \cite{CANTONI1976275}: symmetric centrosymmetric matrices of order $\n$ admits an orthonormal basis of eigenvectors with:
\begin{enumerate}
    \item $\lceil\frac{n}{2}\rceil$ symmetric eigenvectors, 
    \item $\lfloor \frac{n}{2} \rfloor$ skew-symmetric eigenvectors
\end{enumerate}

\subsection{Motif structure of cyclic permutation dynamics with unit spectral radius}

Following Remark~\ref{rmk:ntau}, we focus on the case where $\tau = n$ below. In the special case when the dynamical coupling is a cyclic permutation, the corresponding matrix $\BQ$ is \textit{circulant}. A circulant matrix is a specific type of Toeplitz matrix in which each row is a cyclic shift of the previous row, with elements shifted one position to the right.

The cyclic permutation can equivalently be expressed as the linear map:
    \begin{align}
    \label{eqn.c}
        \BC: \mathbb{R}^n &\rightarrow \mathbb{R}^n \nonumber \\ 
        \left(a_k\right)_{k=1}^n &\mapsto \left(b_k\right)_{k=1}^n := \left(a_{k-1 \pmod n}\right)_{k=1}^n
    \end{align}
    Then by construction we have:
    \begin{enumerate}[label=\textbf{C.\arabic*}]
        \item \label{c.1} $\BC^\alpha \cdot \BC^\beta = \BC^{\alpha + \beta \pmod n}$ .
        \item \label{c.2} $\left(\BC^\top\right)^i = \BC^{n-i \pmod n}$ 
    \end{enumerate}
\begin{theorem}
\label{thm:circulantQ}
    Let $\BW \in \mathbb{M}_{n\times n}\left(\mathbb{R}\right)$ be the dynamical coupling matrix of a reservoir system $R=(\BW,\Bw,h)$. If $\BW$ is a cyclic permutation with spectral radius $\rho = 1$. Then the 
metric tensor $\BQ_1 := \BQ_{\rho = 1}$ of
the reservoir kernel  is circulant. Moreover, there exists an orthogonal basis such that $R$ has $\lceil\frac{n}{2}\rceil$ symmetric motifs and $\lfloor \frac{n}{2} \rfloor$ skew-symmetric motifs.
\end{theorem}

\begin{proof}
    Let $r_i$ denote the $i^{th}$ row of the matrix $\BQ := \BQ_1$. It suffices to show that $r_i = \BC^{i-1} \cdot r_1$. By construction, for $j  =1,\ldots, n$:
    \begin{align*}
        \left(r_1\right)_j & = \Bw^\top \BC^{j-1} \Bw \\
        \left(r_i\right)_j & = \Bw^\top \left(\BC^\top\right)^{i-1} \BC^{j-1} \Bw 
    \end{align*}
    By \ref{c.1} and \ref{c.2}, the $j^{th}$ component of $r_i$ can be rewritten as:
    \begin{align}
    \label{eqn.rij}
        \left(r_i\right)_j & = \Bw^\top \left(\BC^\top\right)^{i-1} \BC^{j-1} \Bw \nonumber \\
                & = \Bw^\top \BC^{n-i+1} \BC^{j-1} \Bw= \Bw^\top \BC^{j-i \pmod n} \Bw
    \end{align}
    By Equation \ref{eqn.c} and \ref{c.1}, the $j^{th}$ component of $\BC^{i-1} \cdot r_1$ can be written as:
    \begin{align*}
        \BC^{i-1}\left(r_1\right)_j & = \Bw^\top \BC^{j-(i-1)-1 \pmod n} \Bw = \Bw^\top \BC^{j-i \pmod n} \Bw
    \end{align*}
    which coincides with Equation \ref{eqn.rij} and thus $\BQ_1$ is circulant.

    Circulant matrices are Toeplitz. Since $\BQ_1$ is symmetric circulant, it is symmetric centrosymmetric of order $n$. Thus by \cite{CANTONI1976275}, $\BQ_1$ admits an orthonormal basis spanning the space of space of motif consisting of $\lceil\frac{n}{2}\rceil$ symmetric eigenvectors of $\BQ_1$ and $\lfloor \frac{n}{2} \rfloor$ skew-symmetric eigenvectors of $\BQ_1$.
\end{proof}

\section{Linear SCR at unit spectral radius is Fourier transform}
\label{sec:fourier_basis}

This section bridges the theoretical results and numerical experiments by explicitly constructing the Fourier basis matrix corresponding to a linear SCR over $\mathbb{R}$. We present numerical simulations to validate our theoretical findings and outline the components of the numerical experiments in the final section.

Combining the results of the previous sections, we conclude that: At unit spectral radius, the motifs of linear SCR over $\mathbb{R}$ with $n$ neurons of look back window $\tau = n$ are:

\begin{enumerate}[label=\textbf{R.\arabic*}]
    \item \label{r.1} Harmonic, i.e. they correspond to the Fourier basis (Theorem~\ref{thm:harmonic}).
    \item \label{r.2} There are $\lceil\frac{n}{2}\rceil$ symmetric vectors (cosines) and $\lfloor \frac{n}{2} \rfloor$ skew-symmetric vectors (sines) with positive eigenvalues. (Theorem~\ref{thm:circulantQ}).
\end{enumerate}

We now demonstrate the explicit construction of the Fourier basis matrix corresponding to the motifs of a linear SCR at unit spectral radius. 

For a linear reservoir system over $\mathbb{R}$, the motifs must also be real. By conditions \ref{r.1} and \ref{r.2}, each motif must then correspond to either the real \textit{or} imaginary parts of the Fourier basis. More precisely, for $ k = 0, \ldots, \lceil \frac{n}{2} \rceil$ and  $i = 0, \ldots, \tau - 1$ (The ceiling function accounts for Theorem~\ref{thm:circulantQ}.), the real Fourier basis matrix \( \BF = \BF[i, j] \in \mathbb{M}_{\tau \times n}(\mathbb{R}) \) is defined as

\begin{align}
    \label{eqn:fourier}
    \begin{split}
    \BF[i,2k] &= \sqrt{\frac{2}{\tau}} \cos{\frac{2\pi k i}{\tau}} \quad \ldots \quad \text{Even columns} \\
    \BF[i,2k+1] &= \sqrt{\frac{2}{\tau}} \sin{\frac{2\pi k i}{\tau}} \quad \ldots \quad \text{Odd columns} 
    \end{split}
\end{align}

We claim that the motifs of a linear SCR are exactly the first $n$ columns of $\BF$.

In Figure~\ref{fig:fourier_spectra} and Figure~\ref{fig:motif_fourier}, we present the Fourier analysis of motifs of the SCR at unit spectral radius when $\tau = n = 97$.

In Figure~\ref{fig:motif_fourier}, we present examples of $4$ randomly chosen motifs and their corresponding Fourier basis in $\BF$. While some motifs and their corresponding Fourier basis are shifted by a fixed phase of $\pi$ or reflected over the $x$-axis, their Fourier spectra align as illustrated in Figure~\ref{fig:fourier_spectra}. Furthermore, Figure~\ref{fig:fourier_spectra}, we observe that the eigenvalues come in pairs, as discussed in \ref{rmk:pairs}. The motifs may not be strictly harmonic in the classical sense due to the coarseness of the discretization grid (non-integer division of period) (See Remark~\ref{rmk:grid} below).

\begin{figure}
  \begin{subfigure}[t]{.45\textwidth}
    \centering
    \includegraphics[width=\linewidth]{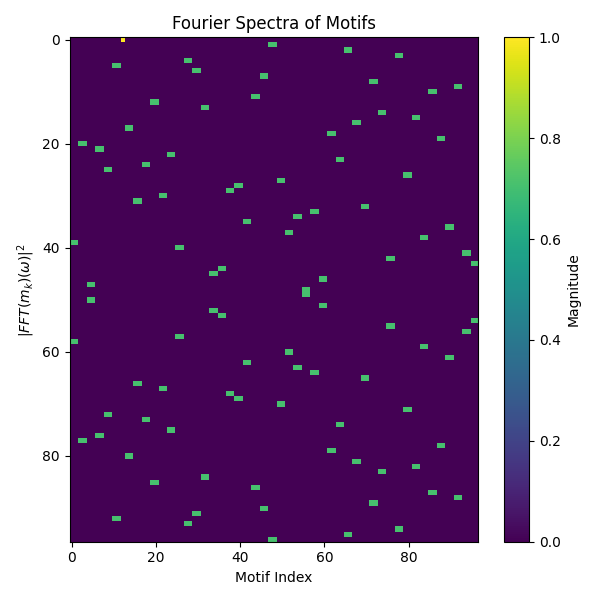}
    \caption{Column-wise FFT of motifs of linear SCR with unit spectral radius.}
    \label{fig:spectra_motif}
  \end{subfigure}
  \hfill
  \begin{subfigure}[t]{.45\textwidth}
    \centering
    \includegraphics[width=\linewidth]{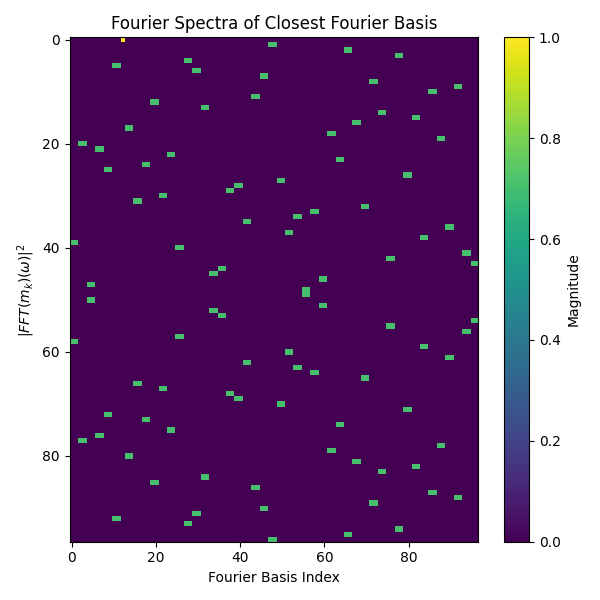}
    \caption{Column-wise FFT of first $n$ columns of $\BF$.}
    \label{fig:spectra_fourier}
  \end{subfigure}

  \medskip

  \begin{subfigure}[t]{.45\textwidth}
    \centering
    \includegraphics[width=\linewidth]{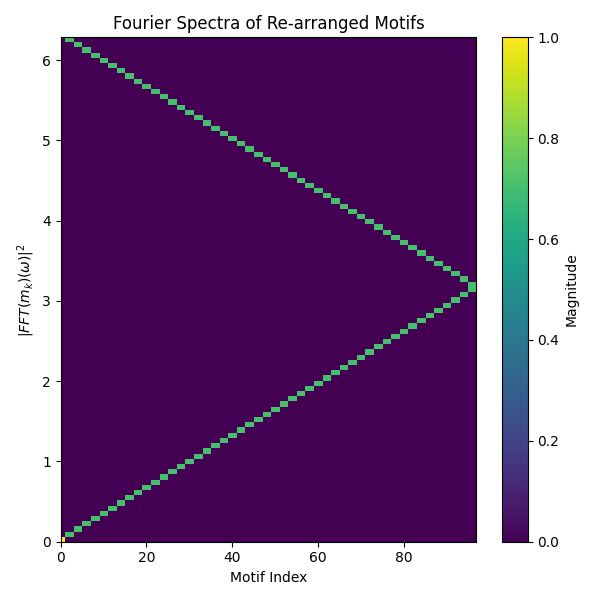}
    \caption{Columns of Fig~\ref{fig:spectra_motif} rearranged according to largest Fourier spectra.}
    \label{fig:rearranged_motif}
  \end{subfigure}
  \hfill
  \begin{subfigure}[t]{.45\textwidth}
    \centering
    \includegraphics[width=\linewidth]{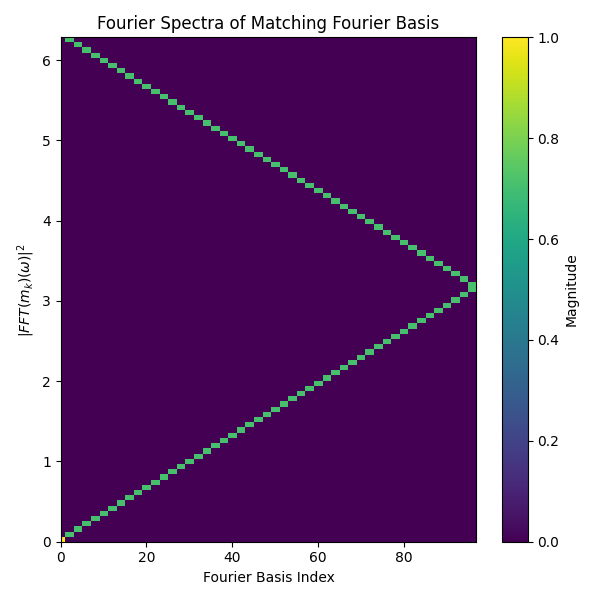}
    \caption{Columns of Fig~\ref{fig:spectra_fourier} rearranged with same indices as Fig~\ref{fig:rearranged_motif}.}
  \end{subfigure}
  \caption[]{Column-wise FFT of motifs of linear SCR with $\rho = 1$ and the column-wise FFT of $\BF$. The first row shows the Fourier spectra in the original form and the second row has their columns rearranged with the same shuffling indices.}
\label{fig:fourier_spectra}
\end{figure}

\begin{figure}[ht!]
	\centering	{\includegraphics[width=\textwidth]{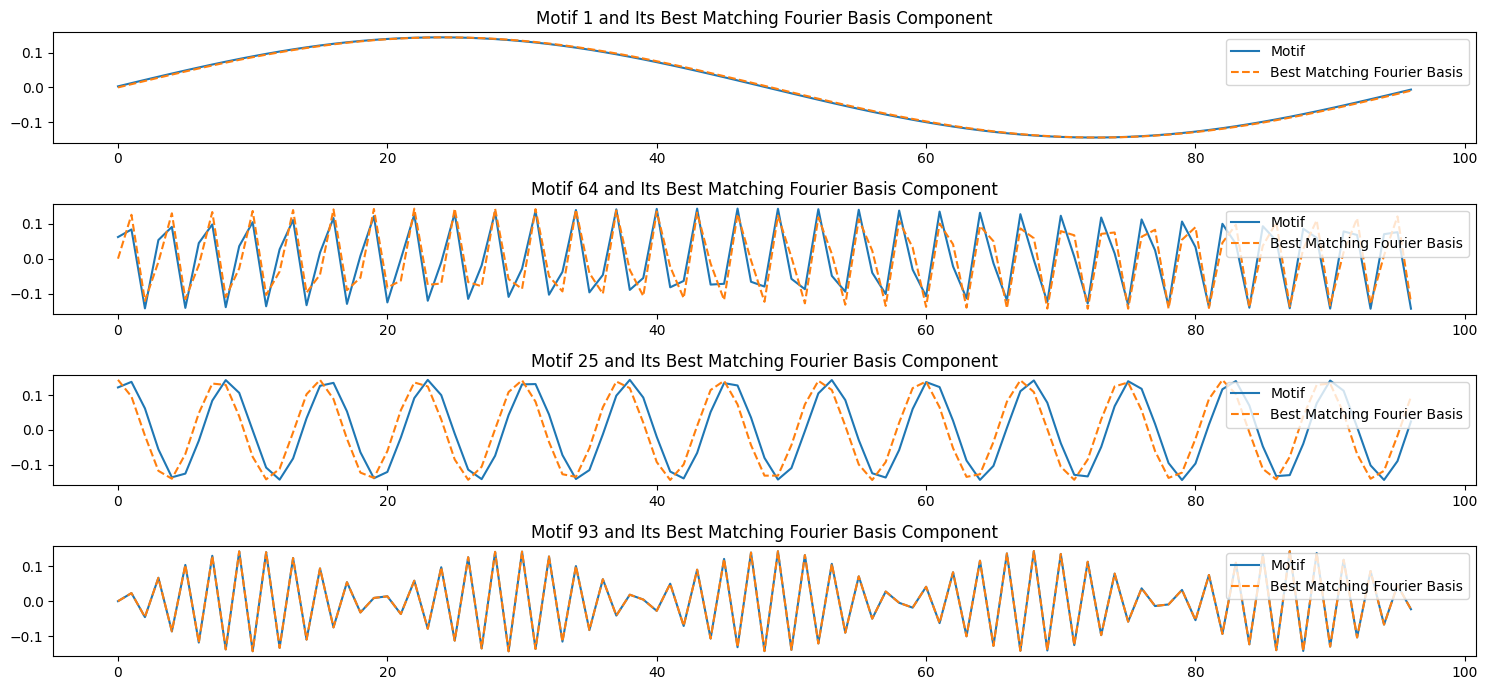}}
	\caption{Example of 4 randomly chosen motifs and their corresponding Fourier basis. Notice some are off by a phase of $\pi$.}
	\label{fig:motif_fourier}
\end{figure}

\clearpage

\begin{remark}
\label{rmk:grid}
In practice, motifs represent harmonic functions sampled at specific frequencies. This explains why certain motifs may not visually appear as harmonic functions (e.g., motif 93 in Figure~\ref{fig:motif_fourier}), yet their Fourier spectra reveal characteristics consistent with harmonic functions (see Figure~\ref{fig:fourier_spectra}). The effect of sampling is further demonstrated in Figure~\ref{fig:compare_sampling}, which compares the 93rd Fourier basis generated according to Equations~\ref{eqn:fourier} at two different sampling sizes.

On the first plot the sine curve is being sampled a high frequency with 500 sampling points; whereas the sampling points are reduced to $97 = n = \tau$ on the second plot. On the third plot we see that they are generated by the exact same function but with different sampling frequency. Notice that the first plot appears to be harmonic but the second one does not. 

Therefore, While this function may not visually resemble a harmonic function at the default sampling size of $n = 97$, increasing the sampling size to $500$ makes the function progressively align with the visual properties of a harmonic function.

The reason the Fourier spectra align with those of harmonic functions in the computational process is due to the matching of sampling frequencies within the FFT algorithm.
\end{remark}

\begin{figure}[ht!]
	\centering	{\includegraphics[width=\textwidth]{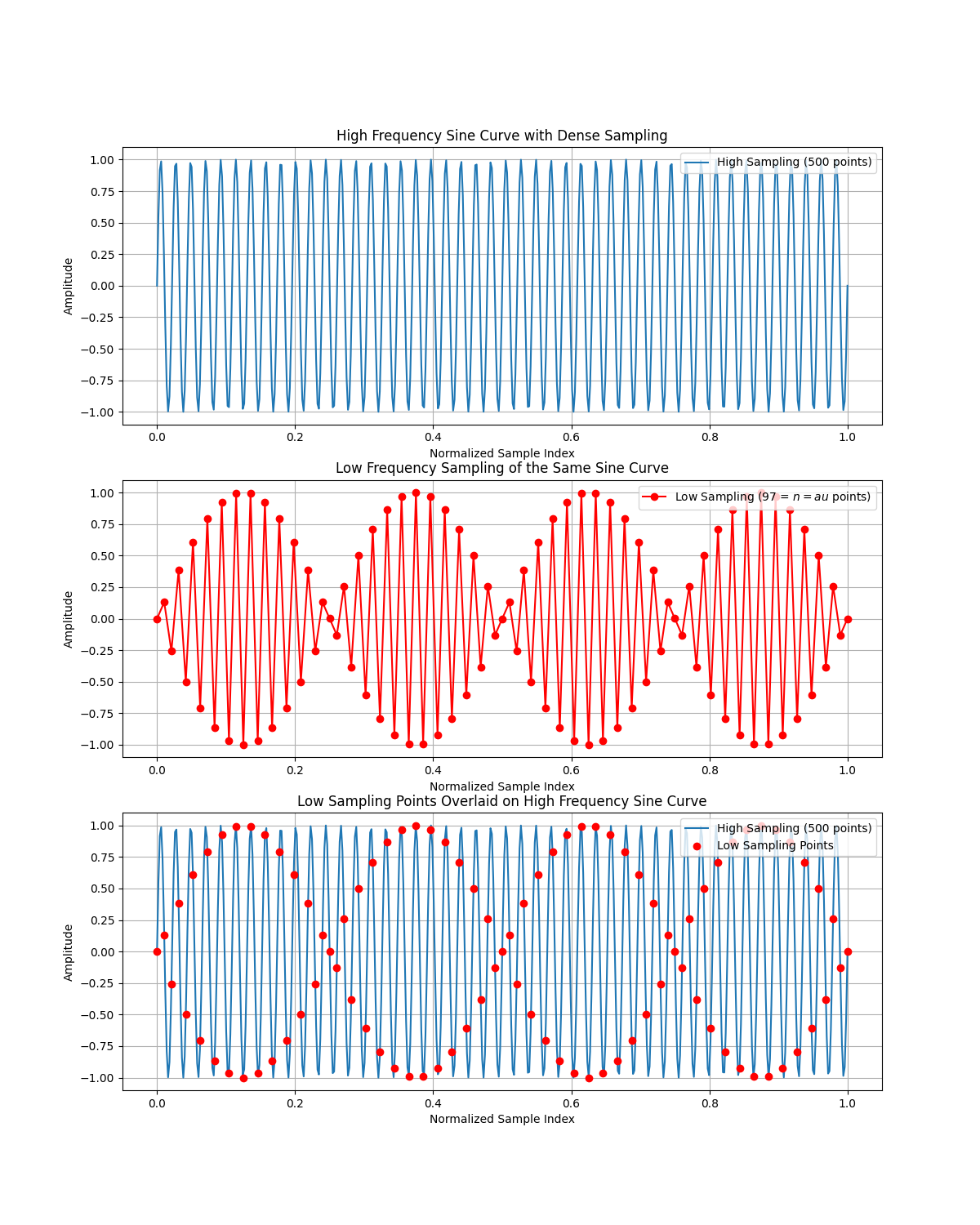}}
	\caption{Plots of Fourier basis number $93$ constructed under Equations~\ref{eqn:fourier} sampled under two different frequencies. }
	\label{fig:compare_sampling}
\end{figure}

\clearpage

\section{Illustrative Examples on Time Series Forecasting}
\label{sec:exp}
We conclude the paper with numerical experiments illustrating our findings. We compare time series forecasting results using the so-called Reservoir Motif Machines (RMM) \cite{tino2024} and the linear SCR. RMM is a simple time series forecasting method based on the feature space representation of time series derived from linear reservoirs and has demonstrated remarkable predictive performance, even surpassing that of more complex transformer models on univariate time series forecasting tasks. A brief exposition of RMM is presented in Section~\ref{sec:linear_motifs}.

We use RMM because it allows us to explicitly define the feature space representation by imposing a set of motifs, rather than having the feature space representation implicitly defined as in classical ESNs, such as SCR. This provides an ideal platform to showcase our theoretical findings: that the feature space representation of the motif space of a linear SCR is the same as that defined by the Fourier basis matrix $\BF$ in Equation \ref{eqn:fourier}.

For reproducibility of the experiments, all experiments are CPU-based and are performed on Apple M3 Max with 128GB of RAM. The source code and data of the numerical analysis is openly available at \texttt{https://github.com/Lampertos/motif\_Fourier}.

We compare the prediction results on univariate time series forecasting across the following models:
\begin{enumerate}
    \item Lin-RMM with SCR motifs and unit spectral radius,
    \item Lin-RMM with Fourier basis motifs as discussed in Section~\ref{sec:fourier_basis}, and
    \item Linear SCR with unit spectral radius.
\end{enumerate}

It is essential to emphasize that this experiment is not intended to showcase the predictive capability of the models, but rather to highlight the similarities in feature space representation between the linear SCR model over $\mathbb{R}$ and the Fourier basis motifs introduced in Section~\ref{sec:fourier_basis}. Consequently, hyperparameters are kept constant throughout the experiments to underscore this feature space comparison. 

The fixed set of hyperparameters for all experiments are as follows:
$r_{\text{in}} = 0.05$ for input weights, $n = 97$ for the number of reservoir neurons, and $\tau = 2n = 194$ for the look-back window length. All models are trained using ridge regression with a ridge coefficient of $10^{-3}$. The prediction horizon is set to $168$.

\subsection{Datasets}
To facilitate the comparison of our results with the state-of-the art, we have used
the same datasets and the same experimental protocols used in the recent time series forecasting
papers \cite{zhou2021informer}. Those are briefly described below for the sake of
completeness.

\paragraph{ETT}
The Electricity Transformer Temperature dataset\footnote{\texttt{https://github.com/zhouhaoyi/ETDataset}}consists of measurements of oil temperature
and six external power-load features from transformers in two regions of China. The data was recorded for two years, and measurements are provided either hourly (indicated by 'h') or every $15$ minutes (indicated by 'm'). In this paper we used \texttt{oil temperature} of the \texttt{ETTh1}, \texttt{ETTm1} dataset for univariate prediction with train/validation/test split being $12$/$4$/$4$ months.

\paragraph{ECL}
The Electricity Consuming Load%
\footnote{\texttt{https://archive.ics.uci.edu/dataset/321/ \\ electricityloaddiagrams20112014}} consists of
hourly measurements of electricity consumption in kWh for 321 Portuguese clients during two years.
In this paper we used client \texttt{MT 320} for univariate prediction. The train/validation/test split is 15/3/4 months.

\paragraph{Weather}
The Local Climatological Data (LCD) dataset%
\footnote{https://www.ncei.noaa.gov/data/local-climatological-data/} consists of hourly measurements
of climatological observations for 1600 weather stations across the US during four years.
The dataset was used for  univariate prediction of the \texttt{Wet Bulb Celcius} variable.

\subsection{Discussion}

From Figure~\ref{fig:performance} and Figure~\ref{fig:mse_rmm} we see that there's virtually no difference between Lin-RMM with unit spectral radius SCR motifs and Lin-RMM with Fourier basis. This confirms our observations in the previous sections. Furthermore Figure~\ref{fig:performance} also confirms that Lin-RMM with unit spectral radius SCR motifs (red bar) has superior performance against classical SCR with unit spectral radius, this affirms the studies in \cite{tino2024}.

\begin{figure}[ht!]
	\centering	{\includegraphics[width=\textwidth]{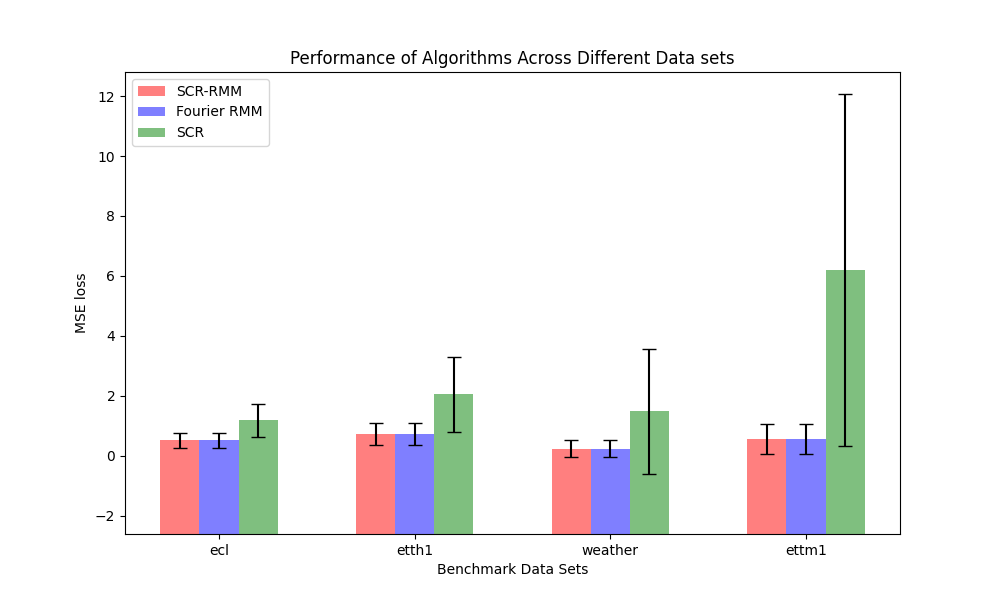}}
	\caption{Statistics of MSE loss of Fourier RMM, unit SCR RMM and unit SCR across different data sets with fixed prediction horizon $168$.}
	\label{fig:performance}
\end{figure}

Amongst RMM's we compare the MSE loss in Figure~\ref{fig:mse_rmm}. Notice that the  difference between the MSE loss between RMM under Fourier motifs and SCR motifs are around $1e-13$ which is negligible compared to the MSE loss of the models against standardized input signals.

\begin{figure}[ht!]
	\centering	{\includegraphics[width=\textwidth]{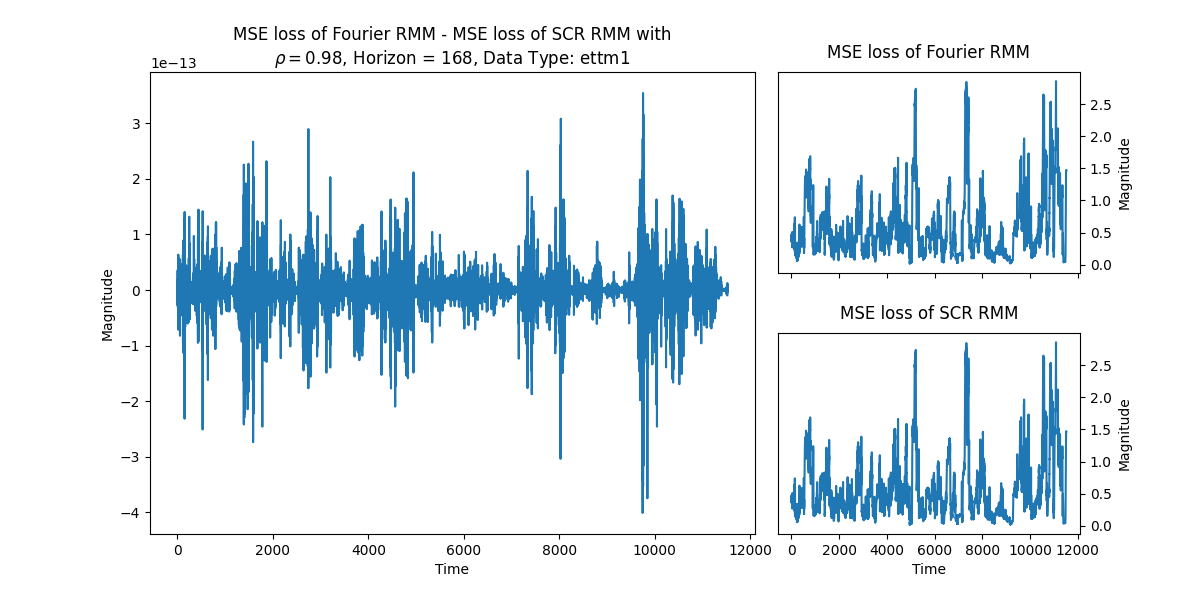}}
	\caption{Comparison of MSE loss between Fourier RMM and unit SCR RMM in \texttt{ETTm2}. }
	\label{fig:mse_rmm}
\end{figure}

\section{Conclusion}
Linear recurrent neural networks (RNN), such as Echo State Networks (ESN) can be thought of as providing feature representations of the input-driving time series in their state space \cite{tino2020dynamical,gonon2022reservoir}.  By endowing the feature (state) space with the canonical dot-product, one can ``reverse engineer" the corresponding inner product in the space of time series (time series kernel) that is defined through the the dot product in the RNN state space \cite{tino2020dynamical}. This in turn helps to shed light on the inner representational schemes employed by the RNN to process the input-driving time series. In particular, the induced (semi-)inner product in the time series space can be theoretically analyzed through eigen-decomposition of the corresponding metric tensor. The eigenvectors (time series motifs) define the projection basis of the induced feature space and the (decay of) eigenvalues its dominant subspace and effective dimensionality. The induced time series kernels by the Simple Cycle Reservoir (SCR) models were shown to be superior (in terms of dimensionality, motif variability, and memory) to several alternative ESN constructions \cite{tino2020dynamical}. 

In this paper we have shown a rather surprising result: When SCR is constructed at the edge of stability, the basis of its induced time series feature space correspond to the well known and widely used basis for signal decomposition - namely the Fourier basis.

This insight also explains the reduction in relative area covered by Fourier representations of SCR motifs observed by \cite{tino2020dynamical}
at the edge of stability. Our results imply that the feature space representation of a linear SCR at unit spectral radius effectively performs a weighted projection onto the Fourier basis. 

This observation is supported by numerical experiments, in which we compared the time series forecasting accuracy of Lin-RMM with the motif space defined by a linear SCR at unit spectral radius and the Fourier basis, respectively.

\clearpage
\newcommand{\etalchar}[1]{$^{#1}$}

\end{document}